\documentclass[11pt,a4paper]{article}

\usepackage[nonatbib,preprint]{nips_2020}

\usepackage[utf8]{inputenc} 
\usepackage[T1]{fontenc}    
\usepackage{hyperref}       
\usepackage{url}            
\usepackage{booktabs}       
\usepackage{amsfonts}       
\usepackage{nicefrac}       
\usepackage{microtype}      
\usepackage{multirow}
\usepackage{xcolor}

\usepackage{amsmath,amsfonts,amssymb}
\DeclareMathOperator*{\argmax}{arg\,max}

\usepackage{enumitem}
\usepackage{graphicx}
\usepackage{wrapfig}
\usepackage{dsfont}

\usepackage{tikz}
\usepackage{pgfplots}
\usepackage{amsthm}
\newtheorem{theorem}{Theorem}
\newtheorem{definition}{Definition}
\usetikzlibrary{arrows,positioning,automata,calc,shapes}
\pgfplotsset{compat=newest, scaled z ticks=false} 
\pgfplotsset{plot coordinates/math parser=false}
\newlength\figureheight 
\newlength\figurewidth
\usepackage{subcaption}

\title{Defense Against Explanation Manipulation}

\author{
  Ruixiang Tang \\
  Texas A\&M University \\
  \texttt{rxtang@tamu.edu} \\
  \And
  Ninghao Liu \\
  Texas A\&M University \\
  \texttt{nhliu43@tamu.edu} \\
  \And
  Fan Yang \\
  Texas A\&M University \\
  \texttt{nacoyang@tamu.edu} \\
  \And
  Na Zou \\
  Texas A\&M University \\
  \texttt{nzou1@tamu.edu} \\
  \And
  Xia Hu \\
  Rice University \\
  \texttt{xia.hu@rice.edu} \\
}

\begin{document}

\maketitle

\begin{abstract}
Explainable machine learning attracts increasing attention as it improves transparency of models, which is helpful for machine learning to be trusted in real applications. However, explanation methods have recently been demonstrated to be vulnerable to manipulation, where we can easily change a model's explanation while keeping its prediction constant. To tackle this problem, some efforts have been paid to use more stable explanation methods or to change model configurations. In this work, we tackle the problem from the training perspective, and propose a new training scheme called Adversarial Training on EXplanations (ATEX) to improve the internal explanation stability of a model regardless of the specific explanation method being applied. Instead of directly specifying explanation values over data instances, ATEX only puts requirement on model predictions which avoids involving second-order derivatives in optimization. As a further discussion, we also find that explanation stability is closely related to another property of the model, i.e., the risk of being exposed to adversarial attack. Through experiments, besides showing that ATEX improves model robustness against manipulation targeting explanation, it also brings additional benefits including smoothing explanations and improving the efficacy of adversarial training if applied to the model. The codes of our work are available at: \url{https://github.com/DefIntpMan}.
\end{abstract}

\section{Introduction}
Despite the significant improvements over traditional approaches in many tasks, deep models are usually criticized as being black-boxes~\cite{du2019techniques, lipton2018mythos, Ribe-etal16whyshould}. To tackle this problem, explanation methods have attracted increasing attention as they provide a tool for understanding how predictions are made by complex models. Methods that produce feature importance maps~\cite{Simonyan-etal13deepInsideCNNsaliency, Smilkov-etal18smoothgrad, Sundararajan-etal16integratedGradient} are commonly used as their explanation results are visually intuitive. Furthermore, explanation methods are expected by model developers to diagnose the defects in models~\cite{Guo-etal18lemna, Liu-etal18adversarial, Ribe-etal16whyshould} or abnormalities in data instances~\cite{Fong-Vedaldi17perturbation}.

Nevertheless, recent work discovered that explanation methods, when applied to deep models, are easy to be manipulated~\cite{ghorbani2019interpretation}. That is, we are able to change explanation results without changing model predictions. To tackle this challenge, some efforts~\cite{Yeh-etal19sensitivityOfExplanation} have been paid to improve the stability of explanation methods by using SmoothGrad~\cite{Smilkov-etal18smoothgrad}. In addition,~\cite{Dombrowski-etal19geometry} proposes to replace ReLU activation with the smoothed softplus function to obtain explanations similar to SmoothGrad. However, in the original work~\cite{ghorbani2019interpretation}, the ReLU activation has already been changed to softplus function, while explanations could still be easily manipulated. It thus implies that more effective techniques, besides smoothing explanations or activation functions, are needed in order to stabilize explanation results.

In this work, we try to modify the training process of neural models to improve their inherent robustness against manipulation targeting explanations. We call our approach as Adversarial Training on EXplanations (ATEX). Different from existing efforts which try to select or design a specific explainer that is more stable~\cite{Levine-etal19certifyRobustInterp, Yeh-etal19sensitivityOfExplanation}, ATEX could benefit various existing explanation methods. Different from the method in~\cite{Dombrowski-etal19geometry}, we do not need to change the model architecture. More precisely, through training with augmented data, ATEX regularizes model explanations around data samples. However, explicitly controlling explanation results is computationally prohibitive as it requires a significant amount of computation for second-order gradients. Therefore, ATEX implicitly regularizes explanation, and it only requires information of model predictions (zero-order) and gradients (first-order). 

Besides stabilizing model explanation, ATEX also brings two additional advantages. First, ATEX helps smoothing the feature importance maps of models, even we only use raw gradient instead of SmoothGrad to compute feature importance. Second, ATEX could improve the efficacy of adversarial training on predictions~\cite{Goodfellow-etal14explaining, Madry-etal17deepResistant} which defends against adversarial samples that cause the model to make wrong predictions. Specifically, traditional adversarial training~\cite{Goodfellow-etal14explaining} suffers from the problem that models easily overfit to adversarial examples~\cite{Madry-etal17deepResistant}, and an adversarially trained model turns out to be less robust against adversarial examples crafted with different perturbation directions. In this work, we show that the ineffectiveness of adversarial training stems from the same source as model interpretation instability. As a result, applying ATEX will increase the efficacy of adversarial training.

The key contributions of this work is summarized as below:
\begin{itemize}
    \item We propose a novel adversarial training method called ATEX to increase the stability of explanation of models, so that explanation results are less sensitive to malicious manipulation.
    \item Models trained with ATEX will produce visually smoothed feature importance maps with one-shot gradient, without applying sophisticated approaches such as SmoothGrad.
    \item We discuss the positive correlation between interpretation stability and adversarial training efficacy. Through experiments, we show that the efficacy of adversarial training is improved when applied on models fine-tuned with ATEX.
\end{itemize}
To avoid confusion, we use ``manipulation" to refer to attack on explanation, while ``adversarial attack" still means attack on model prediction. Correspondingly, we use ``ATEX" to mean adversarial training on explanation, while ``adversarial training" alone still means the defense method to improve prediction robustness.

\section{Algorithm Design for Defense Against Manipulation}
\subsection{Explanation Manipulation}
We consider the target neural network model $f: \mathbb{R}^D\rightarrow \mathbb{R}^C$ with softplus non-linearities, where an input instance $\textbf{x}\in \mathbb{R}^D$ is predicted as belonging to class $c^*=\argmax_{c} f_{c}(\textbf{x})$. Given an instance $\textbf{x}$ of interest, the explanation for prediction $f_c(\textbf{x})$ is $\phi(f_c, \textbf{x})$, where $\phi: \mathcal{F}\times \mathbb{R}^D \rightarrow \mathbb{R}^D$ denotes the explanation function. To facilitate discussion, during the development of ATEX, we assume $\phi$ is based on vanilla gradient~\cite{Simonyan-etal13deepInsideCNNsaliency}, i.e., $\phi(f_c, \textbf{x}) = \nabla_{\textbf{x}}f_c(\textbf{x})$. The relative importance score of the $t$-th feature is computed as $|\phi_t(f_c, \textbf{x}')|/\|\phi(f_c, \textbf{x}')\|_1$, which is commonly used in feature importance maps. We will further discuss the scenarios of using other explanation methods in experiments.

The problem of manipulating explanation could be formulated as below~\cite{Ghorbani-etal19fragile}:
\begin{equation}\label{eq:adv_def}
    \begin{split}
        &\argmax_{\textbf{x}'}\,\, d(\phi(f_c, \textbf{x}'), \phi(f_c, \textbf{x})) \\
        s.t.\,\,\,\,\, & \| \textbf{x}'-\textbf{x} \| \le \epsilon_1, \,\, \| f_{c}(\textbf{x}') - f_{c}(\textbf{x}) \| \le \epsilon_2 ,
    \end{split}
\end{equation}
where $d(\cdot, \cdot)$ is the manipulation objective, the first constraint limits perturbation range, and the second constraint preserves prediction. Some typical objectives include:
\begin{itemize}[leftmargin=*]
    \item \textbf{Targeted Attack} controls explanation to be close to certain predefined patterns, where $d(\phi(f_c, \textbf{x}'), \phi(f_c, \textbf{x}))=\sum_{t\in\mathcal{T}} |\phi_t(f_c, \textbf{x}')|/\|\phi(f_c, \textbf{x}')\|_1$ and $\mathcal{T}$ is the set of features that the manipulator wants to highlight.
    \item \textbf{Untargeted Attack} suppresses the contribution of features that were considered as important in clean samples, where $d(\phi(f_c, \textbf{x}'), \phi(f_c, \textbf{x}))=\sum_{t\in\mathcal{T}} - |\phi_t(f_c, \textbf{x}')|/\|\phi(f_c, \textbf{x}')\|_1$ and $\mathcal{T}$ is the set of important features in $\phi(f_c, \textbf{x})$. It is worth noting that $\mathcal{T}$ contains different elements between targeted and untargeted attack scenario.
\end{itemize}

\subsection{A Na\"ive Solution}
Assume $g$ is the new model to train, a straightforward design for adversarial training is to explicitly require explanations to be constant within the neighborhood of each training sample:
\begin{equation}\label{eq:naive_atex}
    \begin{split}
        \min_g\,\, \sum_{\textbf{x}\in \mathcal{X}} [\, \alpha_1 L(g(\textbf{x}), y) + \sum_{\textbf{x}'\sim \mathcal{N}(\textbf{x}, \epsilon)}  [\, \alpha_2 L(g(\textbf{x}'), y) + d(\phi(g_y, \textbf{x}'), \phi(g_y, \textbf{x}))] \,] ,
    \end{split}
\end{equation}
where $L(\cdot, \cdot)$ denotes the instance-level training loss between a prediction and the true label. $\mathcal{N}(\textbf{x}, \epsilon)$ denotes the neighborhood around $\textbf{x}$ within distance of $\epsilon$. The last term in the inner summation explicitly controls the variation of explanation around training samples, while the other terms preserve model prediction performance. Such a design closely mimics the paradigm of traditional adversarial training over model predictions~\cite{Goodfellow-etal14explaining}.

Nevertheless, there are two problems for the formulation in Equation~\ref{eq:naive_atex}. First, since $\phi$ usually relies on first-order partial derivative information, optimization over explanation maps require computing and propagating second-order partial derivatives, which could be costly to iterate over all training samples. Second, the first term in Equation~\ref{eq:naive_atex} assumes that $\phi(g_y, \textbf{x})$ is the ground-truth explanation. However, there could be defects (e.g., noises) in $\phi(g_y, \textbf{x})$, which makes it not a good target to fit. In addition, since we mainly care about the \textit{stability} of explanation, specifying a concrete ground-truth may not be necessary.
\label{explanation manipulation}

\begin{figure}[t]
\centering
 \includegraphics[width=1.02\textwidth]{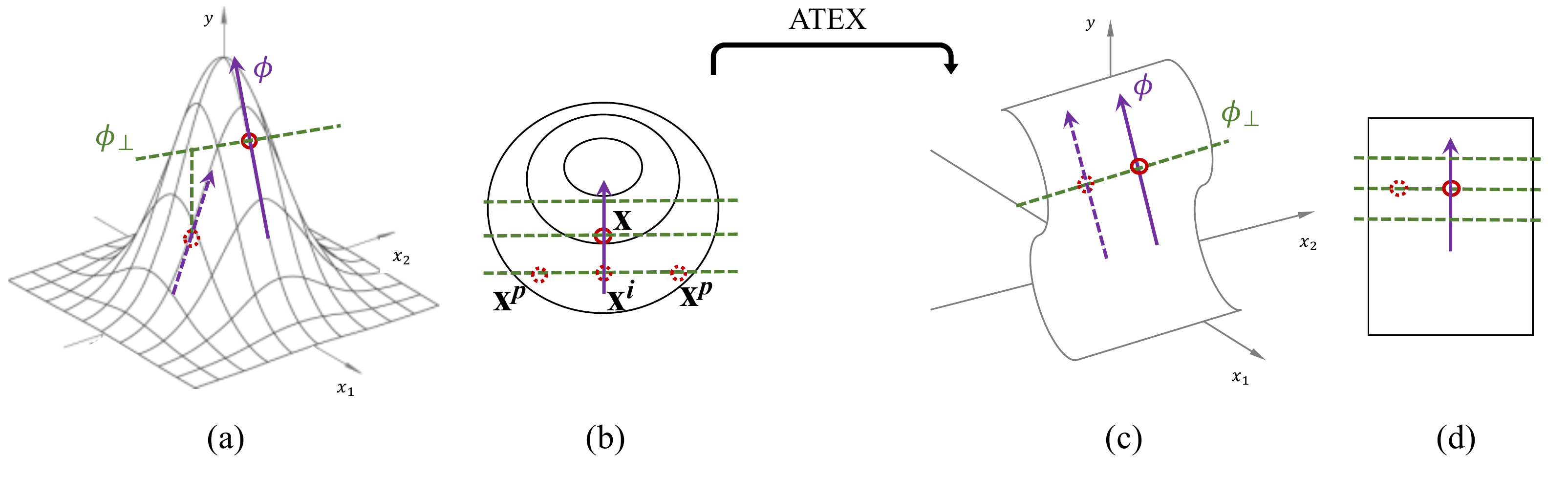}
 \caption{Illustration of explanation stability and ATEX idea. (a): One perspective of why the explanation is prone to be manipulated, i.e., moving an instance along $\phi_{\perp}$ will change its explanation as well as prediction. (b): Illustration of ATEX training process (overhead view from y-axis), where each augmented data instance goes through two rounds of sampling. In the first round, $\textbf{x}^i$ is sampled along explanation direction. In the second round, $\textbf{x}^p$ is sampled perpendicularly to explanation. (c) and (d): A prediction function (an ideal case) that is robust to explanation manipulation.} \label{fig:atex}
\end{figure}

\subsection{Adversarial Training on Explanations (ATEX)}
Let $\textbf{x}' = \textbf{x} + \Delta\textbf{x}$, the sensitivity of gradient-based explanation is $\Delta\phi = \phi(f, \textbf{x}+ \Delta\textbf{x}) - \phi(f, \textbf{x}) = \textbf{H}\Delta\textbf{x} + \mathcal{O}(\|\Delta\textbf{x}\|^2)$, where $\textbf{H}$ is the Hessian matrix and $\textbf{H}_{i,j}=\frac{\partial f}{\partial \textbf{x}_i \partial \textbf{x}_j}$. If $f$ is simply a linear model, then $\phi$ is robust to any manipulation since the Hessian matrix is all-zero. However, a hard requirement to eliminate non-linearity in a deep model would reduce its prediction accuracy. The requirement could be relaxed as long as the explanation is stable according to the below definition.
\begin{definition}
We define the stability of explanation around an instance $\textbf{x}$ as:
\begin{equation}
    \min_{\gamma> 0} \max_{\Delta\textbf{x}}\, \| \phi(f, \textbf{x}+ \Delta\textbf{x}) - \gamma \phi(f, \textbf{x}) \|_2  \,.
\end{equation}
\end{definition}
Different from the proposition in~\cite{Ghorbani-etal19fragile}, we assume a positive scaling does not change explanation, as the relative importance between features is not changed. This is why a coefficient $\gamma$ is introduced here. The definition is compatible with the common metrics for explanation similarity such as Spearman correlation and top elements intersection~\cite{Dombrowski-etal19geometry, Ghorbani-etal19fragile}. One form of $f$ that has stable explanation locally around $\textbf{x}$ could be written as $f(\textbf{x})=\sigma(\phi^{\intercal} \textbf{x})$, where the weights are defined with explanation vector and $\sigma: \mathbb{R}\rightarrow \mathbb{R}$ is a monotonically increasing non-linear function. We have $\phi(f, \textbf{x}) = \sigma'(\phi^{\intercal} \textbf{x})\cdot \phi$. Since $\sigma'(\phi^{\intercal} \textbf{x})$ is a scalar, perturbing input with $\Delta \textbf{x}$ only re-scales $\phi$, thus satisfying the definition above if we let $\gamma = \sigma'(\phi^{\intercal} \textbf{x})$.

Considering the definition above, there are two factors to consider in algorithm design: (i) how to set the nonlinear function $\sigma$; (ii) how to regularize $f$ for stable explanation. The high-level idea of ATEX is illustrated in Figure~\ref{fig:atex}. ATEX is a fine-tuning process given the target model $f$. The formal loss function of ATEX is: $\min_g \sum_{\textbf{x}\in \mathcal{X}} J(g, f, \textbf{x})$, where
\begin{equation}\label{eq:atex}
    J(g, f, \textbf{x}) = L(g(\textbf{x}), f(\textbf{x})) + \alpha \sum_{\textbf{x}^i \sim \mathcal{I}(\textbf{x})} \sum_{\textbf{x}^p \sim \mathcal{P}(\textbf{x}^i)} L(g(\textbf{x}^p), f(\textbf{x}^i)).
\end{equation}
The first term is the distillation loss, and the second term could be seen as a regularizer. Given a seed instance $\textbf{x}\in \mathcal{X}$ from the dataset, two additional sampling process is conducted. In Equation~\ref{eq:atex}, the outer summation generates a set of samples, denoted as $\mathcal{I}(\textbf{x})$, along the explanation direction of $\textbf{x}$. That is,
\begin{equation}
    \textbf{x}^i = \textbf{x} + \delta_1 \phi(f, \textbf{x})/\|\phi(f, \textbf{x})\|_2, \,\,\, -\Delta_1 \le \delta_1 \le \Delta_1 ,
\end{equation}
where $\delta_1$ denotes the shift distance, and $\Delta_1$ is a hyperparameter. To guarantee that we are sampling along a representative explanation direction on the prediction function surface, here we use SmoothGrad~\cite{Smilkov-etal18smoothgrad} to compute $\phi$ in order to remove noise. The inner summation generates samples, denoted as $\mathcal{P}(\textbf{x}^i)$, along the perpendicular direction of explanation $\phi(f, \textbf{x})$. Specifically,
\begin{equation}
    \textbf{x}^p = \textbf{x}^i + \delta_2 \phi_{\perp}(f, \textbf{x})/\|\phi_{\perp}(f, \textbf{x})\|_2 \,\,\,, -\Delta_2 \le \delta_2 \le \Delta_2 ,
\end{equation}
where $\phi_{\perp}$ denotes the perpendicular direction to $\phi$. To compute $\phi_{\perp}$, we first generate a random perturbation $\textbf{u}\sim U(\textbf{0}, \Delta_2)$, and $\phi_{\perp} = \textbf{u} - \phi \cdot \langle \textbf{u}, \phi \rangle/\|\phi\|^2_2$. Here $U$ denotes uniform distribution. The rationale behind moving samples along $\phi_{\perp}$ is that, restricting these samples to have the same prediction as $f(\textbf{x}^i)$ implicitly requires the local explanation to be fixed at $\phi$. As shown in the right half of Figure~\ref{fig:atex}.


\section{Explanation Stability vs Adversarial Training Efficacy}\label{sec:adv_stable}
One of the best known adversarial training method is robust optimization~\cite{Madry-etal17deepResistant}. The goal is to approximately solve:
$
    \min_f \mathbb{E}[\max_{\textbf{x}'\in \mathcal{N}(\textbf{x}, \epsilon)} L(f(\textbf{x}'), y)].
$
The inner maximization problem is usually solved through attacking algorithms such as FGSM~\cite{Goodfellow-etal14explaining} and PGD~\cite{Kurakin-etal17atScale}, where $\textbf{x}'$ can be seen as the most threatening adversarial sample as it maximizes the loss. The outer problem trains model parameters to minimize the loss.

\begin{wrapfigure}{r}{0.38\textwidth}
\centering
 \includegraphics[width=0.315\textwidth]{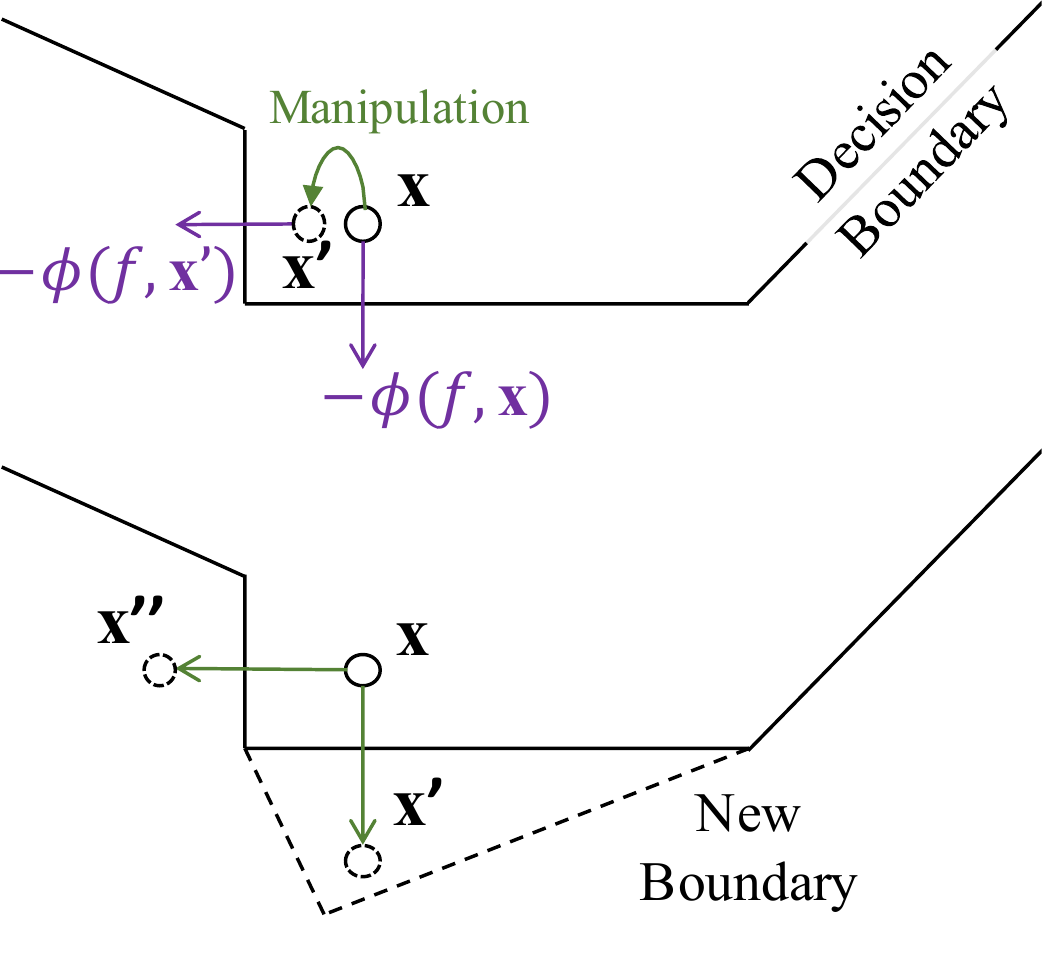} \label{fig:relation}
\end{wrapfigure}

One issue for the above method is that, simply defending against the most threatening adversarial sample is not enough to guarantee prediction robustness. First, other adversarial samples, although leading to smaller losses, could still exist. Second, more adversarial samples could be discovered by using different attacking algorithms. An illustration of such a risk is shown in the lower part of the right figure. Suppose $\textbf{x}'$ is the adversarial sample by perturbing $\textbf{x}$. A new decision boundary is learned via certain defense method, so that $\textbf{x}'$ can no longer fool model prediction. However, it is still possible to perturb $\textbf{x}$ towards other directions (e.g., to $\textbf{x}''$). This prediction is also under the risk of having its explanation been manipulated, as shown in the upper part of the figure. A relation between explanation and adversarial perturbation can be proven as below:
\begin{theorem}
Given a data instance $\textbf{x}_0$, let explanation $\phi(f_c, \textbf{x}_0)$ be defined using vanilla gradient~\cite{Simonyan-etal13deepInsideCNNsaliency}, and adversarial perturbation $\boldsymbol\delta$ be crafted using FGSM~\cite{Kurakin-etal17atScale} without the additional $sign()$ operation, then we have $\phi(f_c, \textbf{x}_0) \propto -\boldsymbol\delta$. The proof can be found in supplementary material. 
\end{theorem}
\begin{proof}
According to~\cite{Simonyan-etal13deepInsideCNNsaliency}, $f_c(\textbf{x}_0)$ is explained via linear approximation by computing its first-order Taylor expansion:
\begin{equation}
    f_c(\textbf{x}) \approx f_c(\textbf{x}_0) + \textbf{w}_c^T \cdot (\textbf{x}-\textbf{x}_0)
\end{equation}
where $\phi(f_c, \textbf{x})=\textbf{w}_c = \nabla_{\textbf{x}} f_c(\textbf{x}_0)$.

On the other and, in FGSM~\cite{Goodfellow-etal14explaining}, let $L(f(\textbf{x}_0), y)$ be the cross entropy loss, and the target label to be $c$, then 
\begin{equation}
\begin{split}
    \boldsymbol\delta &= \nabla_{\textbf{x}} L(f(\textbf{x}_0), c) \\
    &= \nabla_{\textbf{x}} \Big( -\sum_y \mathds{1}[y=c] \log f_y(\textbf{x}_0) \Big) \\
    &= -\nabla_{\textbf{x}} \log f_c(\textbf{x}_0) \\
    &= -\frac{1}{f_c(\textbf{x}_0)} \nabla_{\textbf{x}} f_c(\textbf{x}_0),
\end{split}
\end{equation}

where $\frac{1}{f_c(\textbf{x}_0)}$ is a scalar. Therefore, we have $\phi(f_c, \textbf{x}) \propto -\boldsymbol\delta$.
\end{proof}
Therefore, if a prediction $f_c(\textbf{x})$ does not have a stable explanation, then this prediction could potentially be attacked towards multiple directions, thus requiring doing more iterations of adversarial training. In experiments, we will show that ATEX could improve the efficacy of adversarial training in each iteration.

\section{Experiments}
The experimental results here demonstrate the efficacy of ATEX on several aspects. Specifically, in Section~\ref{sec:exp_stable}, we show how ATEX could improve interpretation stability. In Section~\ref{sec:exp_attack}, we show that ATEX could mitigate noises in feature importance maps generated by vanilla gradient interpretation. In Section~\ref{sec:exp_assess}, we further demonstrate that ATEX can accelerate the adversarial training process, which ATEX requires fewer adversarial training samples to obtain a decent defense performance.

\begin{table}[t!]
	\centering
      \begin{tabular}[0.1\textwidth]{cccc}
      \toprule
      \multicolumn{1}{c}{$\epsilon_1$} & \multicolumn{1}{c}{\textbf{Model Accuracy}} & \multicolumn{1}{c}{\textbf{Rank Correlation}} & \multicolumn{1}{c}{\textbf{Top-k Intersection}}  \\
      \multicolumn{1}{c}{} & \multicolumn{1}{c}{} & \multicolumn{1}{c}{(ATEX / \cite{Dombrowski-etal19geometry})} & \multicolumn{1}{c}{(ATEX / \cite{Dombrowski-etal19geometry})} \\
      \hline 
      $0.02$ & $0.884$  & $0.766/0.708$  & $0.747/0.674$  \\ 
      $0.04$ & $0.878$  & $0.715/0.622$  & $0.717/0.574$  \\ 
      $0.08$ & $0.870$  & $0.686/0.536$  & $0.702/0.484$  \\ 
      \bottomrule
      \end{tabular}
     \vspace{2pt}
	\caption{Defense against untargeted explanation manipulation on FashionMNIST.} \label{table:fashion_untar}
\vspace{0pt}
\end{table}

\begin{table}[t!]
	\centering
      \begin{tabular}[0.1\textwidth]{cccc}
      \toprule
      \multicolumn{1}{c}{$\epsilon_1$} & \multicolumn{1}{c}{\textbf{Model Accuracy}} & \multicolumn{1}{c}{\textbf{Rank Correlation}} & \multicolumn{1}{c}{\textbf{Top-k Intersection}}  \\
      \multicolumn{1}{c}{} & \multicolumn{1}{c}{} & \multicolumn{1}{c}{(ATEX / \cite{Dombrowski-etal19geometry})} & \multicolumn{1}{c}{(ATEX / \cite{Dombrowski-etal19geometry})} \\
      \hline 
      $0.02$ & $0.887$  & $0.746/0.698$  & $0.717/0.671$  \\ 
      $0.04$ & $0.878$  & $0.708/0.618$  & $0.681/0.577$  \\ 
      $0.08$ & $0.867$  & $0.710/0.540$  & $0.667/0.502$  \\ 
      \bottomrule
      \end{tabular}
     \vspace{2pt}
	\caption{Defense against targeted explanation manipulation on FashionMNIST.} \label{table:fashion_tar}
\vspace{0pt}
\end{table}

\begin{table}[t!]
	\centering
      \begin{tabular}[0.1\textwidth]{cccc}
      \toprule
      \multicolumn{1}{c}{$\epsilon_1$} & \multicolumn{1}{c}{\textbf{Model Accuracy}} & \multicolumn{1}{c}{\textbf{Rank Correlation}} & \multicolumn{1}{c}{\textbf{Top-k Intersection}}  \\
      \multicolumn{1}{c}{} & \multicolumn{1}{c}{} & \multicolumn{1}{c}{(ATEX / \cite{Dombrowski-etal19geometry})} & \multicolumn{1}{c}{(ATEX / \cite{Dombrowski-etal19geometry})} \\
      \hline 
      $0.02$ & $0.988$  & $0.864/0.842$  & $0.760/0.732$  \\ 
      $0.04$ & $0.987$  & $0.825/0.787$  & $0.744/0.709$  \\ 
      $0.08$ & $0.988$  & $0.783/0.705$  & $0.808/0.676$  \\ 
      \bottomrule
      \end{tabular}
     \vspace{2pt}
	\caption{Defense against untargeted explanation manipulation on MNIST.} \label{table:mnist_untar}
\vspace{0pt}
\end{table}

\begin{table}[t!]
	\centering
      \begin{tabular}[0.1\textwidth]{cccc}
      \toprule
      \multicolumn{1}{c}{$\epsilon_1$} & \multicolumn{1}{c}{\textbf{Model Accuracy}} & \multicolumn{1}{c}{\textbf{Rank Correlation}} & \multicolumn{1}{c}{\textbf{Top-k Intersection}}  \\
      \multicolumn{1}{c}{} & \multicolumn{1}{c}{} & \multicolumn{1}{c}{(ATEX / \cite{Dombrowski-etal19geometry})} & \multicolumn{1}{c}{(ATEX / \cite{Dombrowski-etal19geometry})} \\
      \hline 
      $0.02$ & $0.987$  & $0.856/0.842$  & $0.699/0.732$  \\ 
      $0.04$ & $0.988$  & $0.825/0.784$  & $0.719/0.708$  \\ 
      $0.08$ & $0.987$  & $0.785/0.708$  & $0.766/0.678$  \\ 
      \bottomrule
      \end{tabular}
     \vspace{2pt}
	\caption{Defense against targeted explanation manipulation on MNIST.} \label{table:mnist_tar}
\vspace{0pt}
\end{table}

\subsection{Experiment Settings}
\begin{itemize}[leftmargin=*]
\item \textbf{Datasets.} We conduct our experiment on the Fashion-MNIST dataset and MNIST dataset. Fashion-MNIST consists of a training set of 60,000 examples and a test set of 10,000 examples. Each example is a 28 $\times$ 28 gray-scale image with a label from 10 categories. Image pixels of all examples are normalized to $[0, 1]$ range. The classification model has two convolutional layers and two FC layers. We use Adam optimizer to train the model with the cross-entropy loss. MNIST consists of a training set of 60,000 examples and a test set of 10,000 examples. Data properties and preprocessing methods are similar to those of FashionMNIST. The classification model also has two convolutional layers and two FC layers.

\item \textbf{Metrics for Interpretation Similarity.} Following the settings in~\cite{Ghorbani-etal19fragile}, we consider three metrics for quantifying the similarity between two feature importance maps. To measure statistic similarity, we have \textit{Spearman’s rank order correlation} which utilizes rank correlation to compare the similarity, and \textit{Top-k intersection} which compares similarity by the size of intersection of the $k$
most important features. For visual similarity, we adopt the Structural Similarity Index (SSIM), which measures the perceptual difference between two similar images. 

\end{itemize}

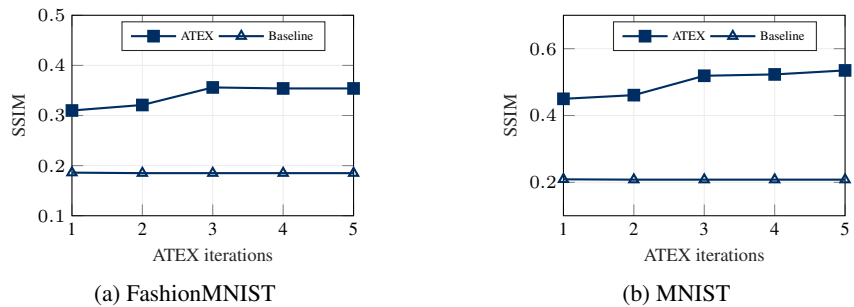
\begin{figure}[t]
        \centering
        \begin{subfigure}[b]{0.43\textwidth}
        \setlength\figureheight{1.05in}
        \setlength\figurewidth{1.55in}
        \centering  \scriptsize
%
%
\definecolor{mycolor1}{rgb}{1,0.65,0}%
\definecolor{mycolor2}{rgb}{1,0.94,0}%
\definecolor{mycolor3}{rgb}{1.0, 0.25, 0.25}%
\definecolor{mycolor4}{rgb}{0.39, 0.58, 0.93}
\definecolor{mycolor5}{rgb}{0., 0.18, 0.39}

\begin{tikzpicture}

\begin{axis}[%
width=0.951\figurewidth,
height=\figureheight,
at={(0\figurewidth,0\figureheight)},
scale only axis,
scaled x ticks=true,
xticklabels={1,2,3,4,5},
xtick={1,2,3,4,5},
xmin=1,
xmax=5,
xlabel style={font=\color{white!15!black}},
xlabel={ATEX iterations},
ymin=0.10,
ymax=0.50,
grid, 
grid style={line width=.15pt, draw=gray!15}, 
ylabel style={font=\color{white!15!black}},
ylabel={SSIM},
axis background/.style={fill=white},
legend columns = 2,
legend style={legend cell align=left, align=left, draw=white!15!black, nodes={scale=0.7}, at={(0.90, 0.97)}},
axis background/.style={fill=white} 
]
\addplot [color=mycolor5, mark=square*, mark options={solid, mycolor5}, thick]
  table[row sep=crcr]{%
1	0.310\\
2	0.321\\
3	0.356\\
4	0.354\\
5	0.354\\
};
\addlegendentry{ATEX}

\addplot [color=mycolor5, mark=triangle, mark options={solid, mycolor5}, thick]
  table[row sep=crcr]{%
1	0.186\\
2	0.185\\
3	0.185\\
4	0.185\\
5	0.185\\
};
\addlegendentry{Baseline}

\end{axis}
\end{tikzpicture}%
            \caption[Network2]%
            {{\small FashionMNIST}}    
        \end{subfigure}
        \hspace{10pt}
        \begin{subfigure}[b]{0.43\textwidth}
        \setlength\figureheight{1.05in}
        \setlength\figurewidth{1.55in}
        \centering  \scriptsize
%
%
\definecolor{mycolor1}{rgb}{1,0.65,0}%
\definecolor{mycolor2}{rgb}{1,0.94,0}%
\definecolor{mycolor3}{rgb}{1.0, 0.25, 0.25}%
\definecolor{mycolor4}{rgb}{0.39, 0.58, 0.93}
\definecolor{mycolor5}{rgb}{0., 0.18, 0.39}

\begin{tikzpicture}

\begin{axis}[%
width=0.951\figurewidth,
height=\figureheight,
at={(0\figurewidth,0\figureheight)},
scale only axis,
scaled x ticks=true,
xticklabels={1,2,3,4,5},
xtick={1,2,3,4,5},
xmin=1,
xmax=5,
xlabel style={font=\color{white!15!black}},
xlabel={ATEX iterations},
ymin=0.10,
ymax=0.70,
grid, 
grid style={line width=.15pt, draw=gray!15}, 
ylabel style={font=\color{white!15!black}},
ylabel={SSIM},
axis background/.style={fill=white},
legend columns = 2,
legend style={legend cell align=left, align=left, draw=white!15!black, nodes={scale=0.7}, at={(0.90, 0.97)}},
axis background/.style={fill=white} 
]
\addplot [color=mycolor5, mark=square*, mark options={solid, mycolor5}, thick]
  table[row sep=crcr]{%
1	0.450\\
2	0.461\\
3	0.519\\
4	0.523\\
5	0.535\\
};
\addlegendentry{ATEX}

\addplot [color=mycolor5, mark=triangle, mark options={solid, mycolor5}, thick]
  table[row sep=crcr]{%
1	0.209\\
2	0.208\\
3	0.208\\
4	0.208\\
5	0.208\\
};
\addlegendentry{Baseline}

\end{axis}
\end{tikzpicture}%
            \caption[Network2]%
            {{\small MNIST}}    
        \end{subfigure}
        \caption{Quantitative evaluation of interpretation smoothness effect.} 
        \label{fig:quant_smooth}
\end{figure}

\begin{figure}[t]
\centering
 \includegraphics[width=1.00\textwidth]{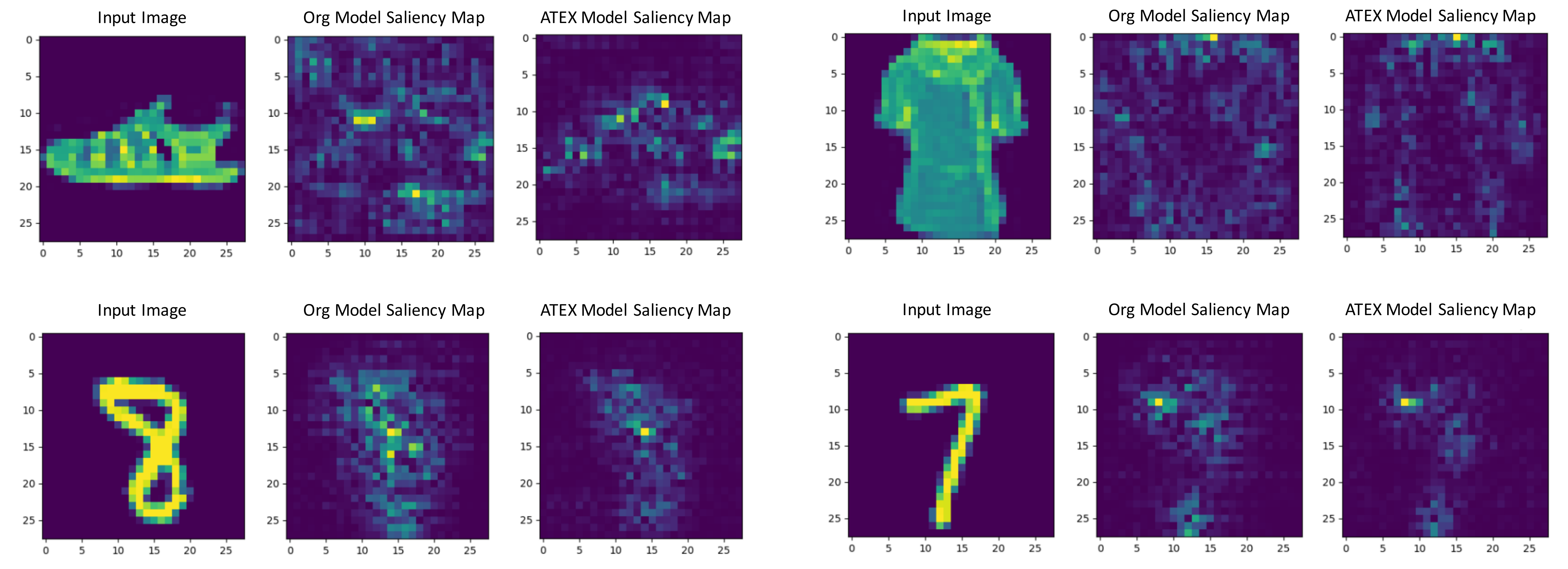}
 \caption{Gradient explanation map produced from the original network and the network trained with ATEX. Three images form a case, which consists of an input, a gradient explanation from the original network, and a gradient explanation from ATEX-trained network. } \label{fig:case}
\end{figure}

\subsection{Defense Performance Against Explanation Manipulation Attack}\label{sec:exp_stable}
In this section, we conduct experiments to measure the interpretation stability of models after applying ATEX. To manipulate explanations, we adopt the two explanation attack approaches introduced in Section~\ref{explanation manipulation}. For targeted attack, we manage to increase model's attention in a predefined region with a size of 5$\times$5 pixels, which are determined randomly in runtime. For untargeted attack, we suppress the contribution of the 50 most important pixels in original samples. Due to the piecewise-linear property~\cite{Ghorbani-etal19fragile} of deep models that use ReLU as activation function, attacking methods that rely on Hessian matrices will not work since second-order gradients are zero. Hence, in this work, we replace ReLU activation with smoothed softplus activation when training models, so~\cite{Dombrowski-etal19geometry} can be seen as the baseline method. Subsequent steps such as generating explanations, manipulation samples, and applying defense, are all implemented on softplus activated models. 

Results are summarized in Table~\ref{table:fashion_untar}$\sim$ Table~\ref{table:mnist_tar}. Compared with the baseline method, we see that ATEX improves the stability of interpretation, in terms of both Rank Correlation and Top-k Intersection metrics. The relative improvement is more significant as the attack magnitude $\epsilon_1$ increases. A larger $\epsilon_1$ means a greater manipulation range ($\Delta_1$ and $\Delta_2$ are set to be equal to $\epsilon_1$). The model prediction accuracy will be slightly affected on FashionMNIST, but remains consistent on MNIST.

\subsection{Qualitative Assessment of Explanation}\label{sec:exp_assess}
In this part, we show that ATEX helps reducing noises in interpretation feature maps, even when we only use vanilla gradient~\cite{Simonyan-etal13deepInsideCNNsaliency} as the interpretation method. 
We choose SmoothGrad~\cite{Smilkov-etal18smoothgrad} as the reference method, because SmoothGrad can reduces the noise in sensitivity maps, and we use SmoothGrad to provide direction to generate $\textbf{x}^i$ in ATEX. In our experiment, we run SmoothGrad on normally training models without applying ATEX. Specifically, we add pixel-wise Gaussian noise to 100 copies of each test image and compute the average of vanilla gradients to get feature maps. In comparison, after running ATEX for $5$ iterations, we use vanilla gradient to produce feature importance maps directly for test images. The baseline feature maps are obtained by vanilla gradient on normally trained models. We expect the interpretation results of ATEX to be more similar to Smoothgrad than baseline results. 
This is validated in Figure~\ref{fig:quant_smooth}, as ATEX achieve higher SSIM scores than the baseline results. We also show the explanation results in Figure~\ref{fig:case}. We could observe that the noise level is significantly reduced in the feature maps after applying ATEX training to models, even though we only use vanilla gradient to generate feature maps. It thus indicates that models trained with ATEX are more focused on the objects in input.

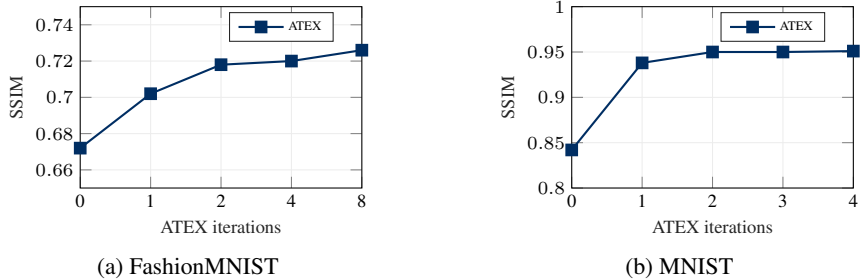
\begin{figure}[t]
        \centering
        \begin{subfigure}[b]{0.43\textwidth}
        \setlength\figureheight{0.95in}
        \setlength\figurewidth{1.55in}
        \centering  \scriptsize
%
%
\definecolor{mycolor1}{rgb}{1,0.65,0}%
\definecolor{mycolor2}{rgb}{1,0.94,0}%
\definecolor{mycolor3}{rgb}{1.0, 0.25, 0.25}%
\definecolor{mycolor4}{rgb}{0.39, 0.58, 0.93}
\definecolor{mycolor5}{rgb}{0., 0.18, 0.39}

\begin{tikzpicture}

\begin{axis}[%
width=0.951\figurewidth,
height=\figureheight,
at={(0\figurewidth,0\figureheight)},
scale only axis,
scaled x ticks=true,
xticklabels={0,1,2,4,8},
xtick={1,2,3,4,5},
xmin=1,
xmax=5,
xlabel style={font=\color{white!15!black}},
xlabel={ATEX iterations},
ymin=0.65,
ymax=0.75,
grid, 
grid style={line width=.15pt, draw=gray!15}, 
ylabel style={font=\color{white!15!black}},
ylabel={SSIM},
axis background/.style={fill=white},
legend columns = 2,
legend style={legend cell align=left, align=left, draw=white!15!black, nodes={scale=0.7}, at={(0.90, 0.97)}},
axis background/.style={fill=white} 
]
\addplot [color=mycolor5, mark=square*, mark options={solid, mycolor5}, thick]
  table[row sep=crcr]{%
1	0.672\\
2	0.702\\
3	0.718\\
4	0.720\\
5   0.726\\
};
\addlegendentry{ATEX}

\end{axis}
\end{tikzpicture}%
            \caption[Network2]%
            {{\small FashionMNIST}}    
        \end{subfigure}
        \hspace{10pt}
        \begin{subfigure}[b]{0.43\textwidth}
        \setlength\figureheight{0.95in}
        \setlength\figurewidth{1.55in}
        \centering  \scriptsize
%
%
\definecolor{mycolor1}{rgb}{1,0.65,0}%
\definecolor{mycolor2}{rgb}{1,0.94,0}%
\definecolor{mycolor3}{rgb}{1.0, 0.25, 0.25}%
\definecolor{mycolor4}{rgb}{0.39, 0.58, 0.93}
\definecolor{mycolor5}{rgb}{0., 0.18, 0.39}

\begin{tikzpicture}

\begin{axis}[%
width=0.951\figurewidth,
height=\figureheight,
at={(0\figurewidth,0\figureheight)},
scale only axis,
scaled x ticks=true,
xticklabels={0,1,2,3,4},
xtick={1,2,3,4,5},
xmin=1,
xmax=5,
xlabel style={font=\color{white!15!black}},
xlabel={ATEX iterations},
ymin=0.80,
ymax=1.00,
grid, 
grid style={line width=.15pt, draw=gray!15}, 
ylabel style={font=\color{white!15!black}},
ylabel={SSIM},
axis background/.style={fill=white},
legend columns = 2,
legend style={legend cell align=left, align=left, draw=white!15!black, nodes={scale=0.7}, at={(0.90, 0.97)}},
axis background/.style={fill=white} 
]
\addplot [color=mycolor5, mark=square*, mark options={solid, mycolor5}, thick]
  table[row sep=crcr]{%
1	0.842\\
2	0.938\\
3	0.950\\
4	0.950\\
5	0.951\\
};
\addlegendentry{ATEX}

\end{axis}
\end{tikzpicture}%
            \caption[Network2]%
            {{\small MNIST}}    
        \end{subfigure}
        \caption{Efficacy of adversarial training after apply ATEX.} 
        \label{fig:adv_train}
\end{figure}

\subsection{Efficacy of Adversarial Training After Applying ATEX}\label{sec:exp_attack}
We now investigate the correlation between explanation stability and adversarial training efficacy. Our analysis in Section~\ref{sec:adv_stable} demonstrates that stability in explanation can potentially improve the efficacy of adversarial training. In this experiment, given a pretrained classifier, we run ATEX for several iterations. After each iteration, to evaluate the efficacy of adversarial training, we further fine-tune the classifier with adversarial training and then evaluate the robustness of the resultant model against a new round of attack. We adopt FGSM as the approach for both adversarial samples generation. The attack step length $\epsilon$ = 0.1. For the adversarial training, we generate 50,000 FGSM attack samples from training data and combine them with original training data to fine-tune the model. Results are shown in Figure~\ref{fig:adv_train}. The x-axis denotes the number of iterations of ATEX, where $iteration=0$ means pure adversarial training without using ATEX. From the figures, we observe that as we run more iterations of ATEX, the performance of adversarial training also increases. It indicates that ATEX reduces the potential weakness contained in models.

\section{Related Work}
Model explanations could be generally indicated and defined as the information which can help people understand the model behaviors. Typically, those useful information could be some significant features that contribute a lot to model predictions. To effectively extract explanations from models, there are two major methodologies, where the first category is based on instance perturbation~\cite{ribeiro2016should} and the second is based on gradient information~\cite{ancona2017towards}. As for the first category, LIME~\cite{ribeiro2016should} is a representative method, utilizing shallow linear models to approximate the model local behaviors with feature importance scores. Further, SHAP~\cite{lundberg2017unified} unifies and generalizes the perturbation-based method with the aid of cooperative game theory, where each feature would be assigned with a Shapley value for explanation purposes. Some other important methods within this category can also be found in~\cite{bach2015pixel,datta2016algorithmic,ribeiro2018anchors}. As for the second category of methods, explanations are mainly extracted and calculated according to the model gradients. Representative methods can be found in~\cite{selvaraju2017grad,chattopadhay2018grad,sundararajan2017axiomatic,shrikumar2017learning,smilkov2017smoothgrad}, where gradients are used as an indicator for feature sensitivity towards model predictions. In this work, we specifically focus on the second category of methods for generating explanations, and aim to make the gradient-based explanations more robust and stable. 

Although model explanations are useful, it can be fragile and easy to be manipulated under certain circumstances. In~\cite{ghorbani2019interpretation}, the authors showed that the gradient-based explanations can be sensitive to imperceptible perturbations of images, which could lead to the unstructured changes in the generated salience maps. One of the approaches proposed in~\cite{kindermans2019reliability} utilized a constant shift on the target instance to manipulate the explanation salience map, where the biases of the neural network are also changed to fit the original prediction. Besides, parameter randomization~\cite{adebayo2018sanity} and network fine-tuning~\cite{heo2019fooling} are also effective approaches in manipulating explanations. To effectively handle such issue, robust and stable explanations are preferred for model interpretability. In~\cite{Yeh-etal19sensitivityOfExplanation}, the authors rigorously define two concepts for generating smooth explanations (i.e., fidelity and sensitivity), and further propose to optimize these metrics for robust explanation generation. Also, the authors in~\cite{Dombrowski-etal19geometry,Ghorbani-etal19fragile} replace the common ReLU activation function with the softplus function, aiming to smooth the explanations during the model training process. Moreover, utilizing the Lipschitz constant of the explanations to locally lower the sensitivity to small perturbations is another valid methodology to improve the explanation robustness~\cite{alvarez2018robustness,melis2018towards}. Our work will specifically focus on the model training perspective for explanation stability under a relatively general setting. 

Besides manipulation over interpretation, a more well studied domain of machine learning security is adversarial attack and defense on model prediction. Adversarial attack on model prediction refers to perturbing input in order to change its prediction results by the model, even though most of the attacks cannot be perceived by humans~\cite{Goodfellow-etal14explaining, Szeg-etal13intriguing}. Adversarial attack can be categorized into different categories according to the threat model, including untargeted attack vs targeted attack~\cite{Carlini-Wagner17towards}, one-shot attack vs iterative attack~\cite{Kurakin-etal17atScale}, data dependent vs universal attack~\cite{Moosavi-etal17universalPerturb}, perturbation attack vs replacement attack~\cite{Thys-Ranst19surveilanceCamera}. Considering such relation between model explanation and adversarial attack, our work also discuss the potential benefit to the target model with the aid of the explanation stability.

\section{Conclusion}
Despite the unique role in improving transparency for neural networks, interpretation methodologies have recently been shown to be vulnerable to manipulation. That is, malevolent users could slightly perturb the input to change its interpretation result while maintaining prediction output. In this work, we propose a new training method called ATEX, which tries to improve model interpretation robustness against manipulation on input. ATEX does not explicitly control interpretation, but implicitly regularize it via control the predictions around training samples. We also show that interpretation stability is closely related to the potential efficacy of adversarial training, since adversarial attack direction has a strong relation to interpretation. Through experiments, we show that ATEX could stabilize interpretation of model predictions. ATEX also reduce noises in feature importance maps, similar to SmoothGrad, even the maps are obtained with vanilla gradient. In addition, ATEX boosts the efficacy of adversarial training.

Future work could investigate how to detect manipulated inputs, which is more efficient especially on large datasets, instead of retraining models. Another interesting direction is how to improve training with augmented data so that the prediction accuracy on clean samples will not decrease.

%

\newpage


\medskip
\bibliographystyle{abbrv}
\bibliography{nips20_ninghao}

\end{document}